\documentclass[letterpaper,11pt]{article}

\usepackage{amssymb}
\usepackage{amsmath}
\usepackage{amsthm}
\usepackage{color}
\usepackage[footnotesize]{caption}
\usepackage[pdftex]{graphicx}
\usepackage{algorithm}
\usepackage{algpseudocode}
\usepackage{subcaption}
\usepackage{tkz-graph}
\usetikzlibrary{calc}
\usepackage{stackengine}
\def\dashfill{\cleaders\hbox to .6em{-}\hfill}
\newcommand\dashline[1]{\hbox to #1{\dashfill\hfil}}

\usepackage{fancyhdr}
\fancypagestyle{plain}{%
  \fancyhead{} 

  \chead{} 
  \cfoot{\thepage}
}

\usepackage{geometry}

\newtheorem{Lemma}{Lemma}
\newtheorem{Theorem}{Theorem}
\newtheorem{Corollary}{Corollary}

\usepackage{epstopdf}

\begin{document}
\title{Feedforward and Recurrent Neural Networks \\ Backward Propagation and Hessian in Matrix Form}
\author{Maxim Naumov \\ {\small NVIDIA, 2701 San Tomas Expressway, Santa Clara, CA 95050}}
\date{}
\maketitle

\begin{abstract}
In this paper we focus on the linear algebra theory behind feedforward (FNN) and recurrent (RNN) neural networks. We review backward propagation, including backward propagation through time (BPTT). Also, we obtain a new exact expression for Hessian, which represents second order effects. We show that for $t$ time steps the weight gradient can be expressed as a rank-$t$ matrix, while the weight Hessian is as a sum of $t^{2}$ Kronecker products of rank-$1$ and $W^{T}AW$ matrices, for some matrix $A$ and weight matrix $W$. Also, we show that for a mini-batch of size $r$, the weight update can be expressed as a rank-$rt$ matrix. Finally, we briefly comment on the eigenvalues of the Hessian matrix. 
\end{abstract}

\section{Introduction}

The concept of neural networks has originated in the study of human behavior and perception in the 1940s and 1950s \cite{Hebb1949,McCulloch1943,Rosenblatt1958}. Different types of neural networks, such as Hopfield, Jordan and and Elman networks, have been developed and successfully adapted for approximating complex functions and recognizing patterns in the 1970s and 1980s \cite{Elman1990,Hopfield1982,Jordan1986,Werbos1989}. 
\par
More recently, a wide variety of neural networks has been developed, including convolutional (CNNs) and recurrent long short-term memory (LSTMs). These networks have been applied and were able to achieve incredible results in image and video classification \cite{Karpathy2014,Szegedy2014,Taylor2010}, natural language and speech processing \cite{Dean2016,Hannun2014,Mikolov2013,Schmidhuber2015}, as well as many other fields. These new results were made possible by a vast amount of available data, more flexible and scalable software frameworks \cite{Tensorflow2015,cuDNN2014,Keras2015,Torch2014,Caffe2014,CNTK2016} and the computational power provided by the GPUs and other parallel computing platforms \cite{Goodfellow2016,Schmidhuber2015,Pascal2017,Volta2017}. 
\par
A neural network is a function $\Phi: \mathbb{R}^{N} \to \mathbb{R}^{M}$, where $N$ and $M$ is the number of inputs and outputs to the network. It is usually expressed through a repeated composition of affine (linear + constant) functions $\Lambda: \mathbb{R}^{n} \to \mathbb{R}^{m}$ of the form
\begin{equation}
\textbf{y} = W\textbf{x}+\textbf{b}
\end{equation}
and non-linear functions $\textbf{f}: \mathbb{R}^{m} \to \mathbb{R}^{m}$ of the form
\begin{equation}
\textbf{z} = \textbf{f}(\textbf{y}) \phantom{111}
\end{equation}
where weight matrix $W \in \mathbb{R}^{m \times n}$, input vector $\textbf{x} \in \mathbb{R}^{n}$, while bias, intermediate and output vectors $\textbf{b}$, $\textbf{y}$ and $\textbf{z} \in \mathbb{R}^{m}$, respectively. The  function $\textbf{f}(.)$ is typically a component-wise application of a monotonic non-decreasing function $f: \mathbb{R} \to \mathbb{R}$, such as logistic (sigmoid) $f(y) = \frac{1}{1+e^{-y}}$, rectified linear unit (ReLU) $f(y)=\max(0,y)$ or softplus $f(y)=\text{ln}(1+e^{y})$, which is a smooth approximation of ReLU \cite{Nair2010,Glorot2011}.
\par
A neural network can also be thought of as a composition of neurons, which add the weighted input signals $x_{j}$ with bias $b$ and pass the intermediate result through a threshold activation function $f(.)$ to obtain an output $z$, as shown on Fig. 1. These single neurons can be further organized into layers, such as the fully connected layer shown in Fig. 2. If the layers are stacked together, with at least one hidden layer that does not directly produce a final output, we refer to this network as a \textit{deep neural network} \cite{Bishop2006,Hinton1986}. 

\begin{table}[h]
        \centering
        \SetVertexNormal[
                         Shape    = circle,
                 	     LineWidth= 1pt]
		\SetUpEdge[lw   = 1pt,
           	   	   color= black,
           	   	   style=->]
		\begin{tikzpicture}		
   			\Vertex[x=4, y=1.5, L=$f(\sum_{j} w_{j}x_{j} + b)$]{4} 
   			  			
   			\tikzset{VertexStyle/.append style = {minimum size = 3pt, inner sep = 0pt, color=black}}
	   		\Vertex[x=0, y=3, LabelOut, Ldist=-1.0cm, L=$x_{1}$]{1}
			\Vertex[x=0, y=1.5, LabelOut, Ldist=-1.0cm, L=$\vdots$]{2}			
			\Vertex[x=0, y=0, LabelOut, Ldist=-1.0cm, L=$x_{n}$]{3}	   		
			
	   		\Vertex[x=8, y=1.5, LabelOut, Ldist=+0.2cm, L=$z$]{5}      
	
	   		\Edge[label={$w_1$}](1)(4)
	   		\Edge[label={$\vdots$}](2)(4)
	   		\Edge[label={$w_n$}](3)(4)
	   		\Edge[](4)(5)   			   		
		\end{tikzpicture} 
\caption*{Fig. 1: A Single Neuron}    
\end{table}
\setcounter{table}{0} 
\setcounter{figure}{1} 

\begin{table}[h]
        \centering
        \SetVertexNormal[
                         Shape    = circle,
                 	     LineWidth= 1pt]
		\SetUpEdge[lw   = 1pt,
           	   	   color= black,
           	   	   style=->]
		\begin{tikzpicture}		
   			\Vertex[x=5, y=2.7, L=$f(y_{1})$]{4a}    			
   			\Vertex[x=5, y=0.8,   L=$f(y_{m})$]{4c} 
   			  			
   			\tikzset{VertexStyle/.append style = {minimum size = 3pt, inner sep = 0pt, color=black}}
	   		\Vertex[x=0, y=3.5, LabelOut, Ldist=-1.0cm, L=$x_{1}$]{1}
			\Vertex[x=0, y=0,   LabelOut, Ldist=-1.0cm, L=$x_{n}$]{3}	   		
			
	   		\Vertex[x=8, y=2.7, LabelOut, Ldist=+0.2cm, L=$z_{1}$]{5a}
	   		\Vertex[x=8, y=0.8,   LabelOut, Ldist=+0.2cm, L=$z_{m}$]{5c}      
	
			\tikzset{VertexStyle/.append style = {minimum size = 0pt}}	
			\Vertex[x=0, y=1.8, LabelOut, Ldist=-0.1cm, L=$\vdots$]{4b}
			\Vertex[x=5, y=1.8, LabelOut, Ldist=-0.1cm, L=$\vdots$]{2}
			\Vertex[x=8, y=1.8, LabelOut, Ldist=-0.1cm, L=$\vdots$]{5b}	
			\Vertex[x=4, y=3.5, LabelOut, Lpos=90, Ldist=-0.1cm, L={$\textbf{z} = \textbf{f}(W\textbf{x}+\textbf{b})$}]{t}	
	
	   		\Edge[label={$w_{11}$},](1)(4a)
	   		\Edge[label={$w_{1n}$}](3)(4a)
	   		
	   		\Edge[label={$w_{m1}$},](1)(4c)
	   		\Edge[label={$w_{mn}$}](3)(4c)   		
	   		
	   		\Edge[](4a)(5a)
	   		\Edge[](4c)(5c)  
	   		
	   		\draw[red,dashed, ultra thick,rounded corners] (1.5,0) rectangle (6,3.5);	 			   		
		\end{tikzpicture} 
\caption*{Fig. 2: A Single Fully Connected Layer}    
\end{table}
\setcounter{table}{0} 
\setcounter{figure}{2} 

\par
The connections between neurons and layers determine the type of a neural network. In particular, in this paper we will work with feedforward (FNNs) and recursive (RNNs) neural networks with fully connected layers \cite{Hochreiter1997,LeCun1986,Pascanu2013}. We point out that in general CNNs can be expressed as FNNs \cite{Ciresan2011,Krizhevsky2012,LeCun1998}. 
\par
The weights $w_{ij}$ associated with a neural network can be viewed as coefficients of the function $\Phi$ defined by it. In pattern recognition we often would like to find these coefficients, such that the function $\Phi$ approximates a training data set $\mathcal{D}$ according to a loss function $\mathcal{L}$ in the best possible way. The expectation is that when a new previously unknown input is presented to the neural network it will then be able to approximate it reasonably well. The process of finding the coefficients of $\Phi$ is called \textit{training}, while the process of approximating a previously unknown input is called \textit{inference} \cite{Bishop1995,Goodfellow2016}.
\par
The data set $\mathcal{D}$ is composed of data samples $\{ (\textbf{x}^{*}, \textbf{z}^{*}) \}$, which are pairs of known inputs $\textbf{x}^{*} \in \mathbb{R}^{N}$ and outputs $\textbf{z}^{*} \in \mathbb{R}^{M}$. These pairs are often ordered and further partitioned into $q$ disjoint mini-batches $\{ (X^{*}, Z^{*}) \}$, so that $X^{*} \in \mathbb{R}^{N \times r}$ and $Z^{*} \in \mathbb{R}^{M \times r}$. We assume that the total number of pairs is $qr$, otherwise the last batch is padded. Also, we do not consider the problem of splitting data into training, validation and test partitions, that are designed to prevent over-fitting and validate the results. We assume this has already been done, and we are already working with the training data set. 
\par
The choice of the loss function $\mathcal{L}$ often depends on a particular application. In this paper we will assume that it has the following form
\begin{equation}
\mathcal{L} = \frac{1}{q} \sum_{p=1}^{q} \mathcal{E}_{p}(X) = \frac{1}{q} \sum_{p=1}^{q} \frac{1}{r} || Z_{p}^{*} - Z_{p}^{(l)} ||_{F}^{2} 
\label{def_loss}
\end{equation}
where $Z^{*}$ is the correct and $Z^{(l)}$ is the obtained output for a given input $X^{*}$, while $||.||_{F}$ denotes the Frobenius norm. 
\par
In order to find the coefficients of $\Phi$ we must find  
\begin{equation}
\text{arg}\min_{w_{ij}} \mathcal{L}  
\label{min_loss}
\end{equation}
Notice that we are not trying to find weights $w_{ij}$ and bias $b_{i}$ that result in a minimum for a particular data point, but rather on ``average" across the entire training data set. In the next sections we will choose to work with scaled loss $\frac{1}{2}\mathcal{L}$ to simplify the formulas. 
\par
The process of adjusting the weights of the neural network to find the minimum of \eqref{min_loss} is called \textit{learning}. In particular, when making updates to the weights based on a single data sample $r=1$ it is called \textit{online}, on several data samples $r,q>1$ it is called \textit{mini-batch}, and on all available data samples $q=1$ it is called \textit{batch} learning. 
\par 
In practice, due to large amounts of data, the weight updates are often made based on partial information obtained from individual components $\mathcal{E}_{p}$ of the loss function $\mathcal{L}$. Notice that in this case we are essentially minimizing a function $\mathcal{E}_{p}$ across multiple inputs, and we can interpret this process as a form of stochastic optimization. We note that the optimization process makes a pass over the entire data set $p=1,...,q$ making updates to weights $w_{ij}$ before proceeding to the next iteration. In this context, a pass over the training data set is called an \textit{epoch}. 
\par
There are many optimization algorithms with different tradeoffs that can find the minimum of the problem \eqref{min_loss}. Some rely only on function evaluations, many take advantage of the gradient $G_{w}$, while others require knowledge of second-order effects using Hessian $H_{w}$ or its approximation \cite{Boyd2004,Nocedal2006,Spall2003}. The most popular approaches for this problem are currently based on variations of stochastic gradient descent (SGD) method, which relies exclusively on function evaluations and gradients \cite{Bottou2016,Duchi2011,Nesterov1983,Polyak1964}. 
\par
These methods require an evaluation of the partial derivatives of the loss function $\mathcal{L}$ or its components $\mathcal{E}$ with respect to weights $w_{ij}$ (for simplicity we have dropped the subscript $p$ in $\mathcal{E}_{p}$). However, notice that the function $\Phi$ specifying the neural network is not given explicitly and is typically only defined through a composition of affine $\Lambda$ and non-linear $\textbf{f}$ functions. In the next sections we will discuss the process for evaluating $\Phi$ called \textit{forward propagation} and the derivatives of $\mathcal{E}$ called \textit{backward propagation} \cite{Bengio1994,Rumelhart1986,Werbos1989,Werbos1990}.

\section{Contributions}

In this paper we focus on the linear algebra theory behind the neural networks, that often generalizes across many of their types. First, we briefly review backward propagation, obtaining an expression for the weight gradient at level $k$ as a rank-$1$ matrix
\begin{equation}
G_{w}^{(k)} = \textbf{v}^{(k)} \textbf{z}^{(k-1)^{T}}
\end{equation}
for FNNs and rank-$t$ matrix
\begin{equation}
G_{w}^{(k,t)} = \sum_{s=1}^{t} \textbf{v}^{(k,t,s)} \textbf{z}^{(k-1,s)^{T}}
\end{equation}
for RNNs at time step $t$. Here, the yet to be specified vector $\textbf{v}$ is related to $\partial \mathcal{E} / \partial y_{i}$, while vector $\textbf{z}$ is the input to the current layer of the neural network. Therefore, we conclude that for a mini-batch of size $r$, the weight update can be expressed as a rank-$rt$ matrix.
\par
Then, we obtain a new exact expression for the weight Hessian, as a Kronecker product\footnote{
The Kronecker product for $m \times n$ matrix $A$ and $p \times q$ matrix $B$ is defined as a $mp \times nq$ matrix 
$A \otimes B = \left[ \begin{matrix} a_{11}B & ... & a_{1n}B \\ \vdots & \ddots & \vdots \\ a_{m1}B & ... & a_{mn}B  \end{matrix} \right]$.
} 
\begin{equation}
H_{w}^{(k)} = C^{(k)} \otimes \left( \textbf{z}^{(k-1)} \textbf{z}^{(k-1)^{T}} \right) 
\label{contributions_hessian_fnn}
\end{equation}
for FNNs, and as a sum of Kronecker products 
\begin{equation}
H_{w}^{(k,t)} = \sum_{s=1}^{t} \sum_{\zeta=1}^{t} C^{(k,t,s,\zeta)} \otimes \left( \textbf{z}^{(k-1,s)} \textbf{z}^{(k-1,\zeta)^{T}} \right)
\label{contributions_hessian_rnn}
\end{equation}
for RNNs at time step $t$. Here, the yet to be specified matrix $C$ is related to $\partial^{2} \mathcal{E} / \partial y_{i}^{(k)} \partial y_{j}^{(k)}$ and will be shown to have the form $W^{T}AW$, for some matrix $A$ and weight matrix $W$. 
\par
Also, we show that expression for Hessian can be further simplified for ReLU activation function $f(y)=\max(0,y)$ by taking advantage of the fact that $f''(y) = 0$ for $y \neq 0$. 
\par
Finally, using the fact that Kronecker products \eqref{contributions_hessian_fnn} and \eqref{contributions_hessian_rnn} involve a rank-$1$ matrix, we will show that the eigenvalues $\lambda$ of $mn \times mn$ Hessian matrix can be expressed as and related to
\begin{equation}
\lambda(H) = \left\{ 0, || \textbf{z}^{(k-1,s)} ||_{2}^{2} \lambda(C) \right\} 
\end{equation} 
for FNNs and RNNs, respectively. Therefore, we can determine whether Hessian is positive, negative or indefinite by looking only at the eigenvalues of $m \times m$ matrix $C$.

\section{Feedforward Neural Network (FNN)}

The feedforward neural network consists of a set of stacked fully connected layers. The layers are defined by their matrix of weights $W^{(k)}$ and vector of bias $b^{(k)}$. They accept an input vector $\textbf{z}^{(k-1)}$ and produce an output vector $\textbf{z}^{(k)}$ for $k=1,...,l$, with the last vector $\textbf{z}^{(l)}$ being the output of the entire neural network, as shown in Fig. 3. 
\par
In some cases the neural network can be simplified to have sparse connections, resulting in a sparse matrix of weights $W^{(k)}$. However, the connections between layers are always such that the data flow graph is a directed acyclic graph (DAG).    

\begin{table}[h]
        \centering
        \SetVertexNormal[
                         Shape    = circle,
                 	     LineWidth= 1pt]
		\SetUpEdge[lw   = 1pt,
           	   	   color= black,
           	   	   style=->]
		\begin{tikzpicture}
   			\tikzset{VertexStyle/.append style = {minimum size = 5pt}}	
   			\Vertex[x=2, y=3, LabelOut, L=$$]{4a}    			
   			\Vertex[x=2, y=1, LabelOut, L=$$]{4c}
   			
   			\Vertex[x=7, y=2.8, LabelOut, L=$$]{14a}    			
   			\Vertex[x=7, y=1.2, LabelOut, L=$$]{14c} 
   			
   			\Vertex[x=12, y=2.5, LabelOut, L=$$]{24a}    			
   			\Vertex[x=12, y=1.5, LabelOut, L=$$]{24c} 
   			  			
   			\tikzset{VertexStyle/.append style = {minimum size = 3pt, inner sep = 0pt, color=black}}
	   		\Vertex[x=0, y=4, LabelOut, Ldist=-1.0cm, L=$x_{1}^{*}$]{1}
			\Vertex[x=0, y=0, LabelOut, Ldist=-1.0cm, L=$x_{N}^{*}$]{3}	   		
			
	   		\Vertex[x=3.5, y=3, LabelOut, Ldist=+0.1cm, L=$\ldots$]{5a}
	   		\Vertex[x=3.5, y=1, LabelOut, Ldist=+0.1cm, L=$\ldots$]{5c}      
	
			\tikzset{VertexStyle/.append style = {minimum size = 0pt}}	
			\Vertex[x=0, y=2, LabelOut, Ldist=-0.8cm, L=$\vdots$]{4b}
			\Vertex[x=2, y=4, LabelOut, Lpos=90, Ldist=-0.1cm, L={Input $\textbf{z}^{(0)} = \textbf{x}^{*}$}]{t}	
	
	   		\Edge[](1)(4a)
	   		\Edge[](3)(4a)
	   		
	   		\Edge[](1)(4c)
	   		\Edge[](3)(4c)   		
	   		
	   		\Edge[](4a)(5a)
	   		\Edge[](4c)(5c)  
	   		
	   		\draw[red,dashed, ultra thick,rounded corners] (1,0) rectangle (3,4);	
	   		
   			  			
   			\tikzset{VertexStyle/.append style = {minimum size = 3pt, inner sep = 0pt, color=black}}
	   		\Vertex[x=5.5, y=3, LabelOut, Ldist=-1.2cm, L=$z_{1}^{(k-1)}$]{11}
			\Vertex[x=5.5, y=1, LabelOut, Ldist=-1.2cm, L=$z_{n}^{(k-1)}$]{13}	   		
			
	   		\Vertex[x=8.5, y=2.8, LabelOut, Ldist=+0.2cm, L=$z_{1}^{(k)}$]{15a}
	   		\Vertex[x=8.5, y=1.2, LabelOut, Ldist=+0.2cm, L=$z_{m}^{(k)}$]{15c}      
	
			\tikzset{VertexStyle/.append style = {minimum size = 0pt}}	
			\Vertex[x=8.9, y=2, LabelOut, Ldist=-0.1cm, L=$\vdots$]{14b}
			\Vertex[x=4.5, y=2, LabelOut, Ldist=-0.1cm, L=$\vdots$]{15b}	
			\Vertex[x=7, y=4, LabelOut, Lpos=90, Ldist=-0.1cm, L={$\textbf{z}^{(k)} = \textbf{f}(W^{(k)}\textbf{z}^{(k-1)}+\textbf{b}^{(k)})$}]{tt}	
	
	   		\Edge[](11)(14a)
	   		\Edge[](13)(14a)
	   		
	   		\Edge[](11)(14c)
	   		\Edge[](13)(14c)   		
	   		
	   		\Edge[](14a)(15a)
	   		\Edge[](14c)(15c)  	   		
	   		
	   		\draw[red,dashed, ultra thick,rounded corners] (8,0) rectangle (6,4);		   		 		
	   		 			   		
   			  			
   			\tikzset{VertexStyle/.append style = {minimum size = 3pt, inner sep = 0pt, color=black}}
	   		\Vertex[x=10.5, y=2.8, LabelOut, Ldist=-0.8cm, L=$\ldots$]{21}
			\Vertex[x=10.5, y=1.2, LabelOut, Ldist=-0.8cm, L=$\ldots$]{23}	   		
			
	   		\Vertex[x=13.5, y=2.5, LabelOut, Ldist=+0.2cm, L=$z_{1}^{(l)}$]{25a}
	   		\Vertex[x=13.5, y=1.5, LabelOut, Ldist=+0.2cm, L=$z_{M}^{(l)}$]{25c}      
	
			\tikzset{VertexStyle/.append style = {minimum size = 0pt}}	
			\Vertex[x=13.9, y=2, LabelOut, Ldist=-0.1cm, L=$\vdots$]{24b}
			\Vertex[x=12, y=4, LabelOut, Lpos=90, Ldist=-0.1cm, L={Output $\textbf{z}^{(l)}$}]{ttt}	
	
	   		\Edge[](21)(24a)
	   		\Edge[](23)(24a)
	   		
	   		\Edge[](21)(24c)
	   		\Edge[](23)(24c)   		
	   		
	   		\Edge[](24a)(25a)
	   		\Edge[](24c)(25c)  	   		
	   		
	   		\draw[red,dashed, ultra thick,rounded corners] (13,0) rectangle (11,4);

		\end{tikzpicture} 
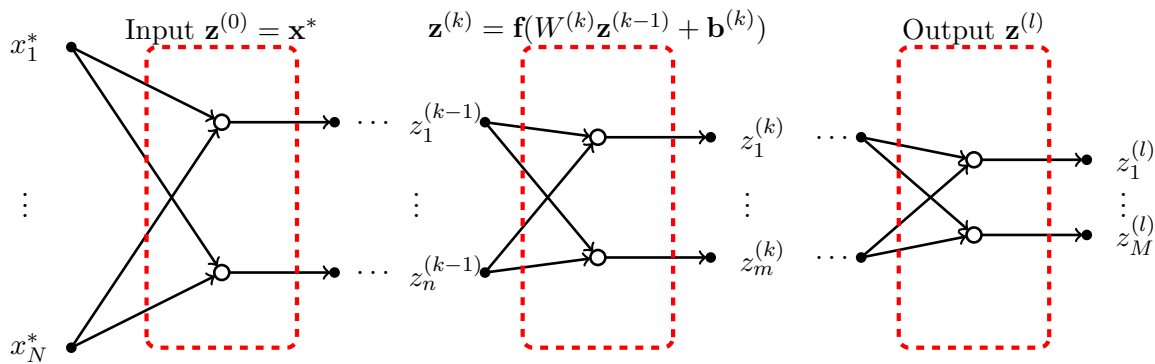
\caption*{Fig. 3: A Sample Feedforward Neural Network (FNN) with $k=1,...,l$ Levels}    
\end{table}
\setcounter{table}{0} 
\setcounter{figure}{3} 

\subsection{Forward Propagation}
Let us assume that we are given an input $\textbf{x}^{*}$, then we can compute an output of the neural network $\textbf{z}^{(l)}$ by repeated applications of the formula
\begin{eqnarray}
\textbf{z}^{(k)} &=& \textbf{f}(\textbf{y}^{(k)})               \label{forwardpropagation1}\\
\textbf{y}^{(k)} &=& W^{(k)}\textbf{z}^{(k-1)}+\textbf{b}^{(k)} \label{forwardpropagation2}
\end{eqnarray}
for $k=1,...,l$ where $\textbf{z}^{(0)} = \textbf{x}^{*}$. This process is called \textit{forward propagation}.

\subsection{Backward Propagation}
Let us assume that we have a data sample $(\textbf{x}^{*},\textbf{z}^{*})$ from the training data set $\mathcal{D}$, as we have discussed in the introduction. Notice that using forward propagation we may also compute the actual output $\textbf{z}^{(l)}$ and the error (associated with scaled loss $\frac{1}{2}\mathcal{L}$) 
\begin{equation}
\mathcal{E} = \frac{1}{2}||\textbf{z}^{*}-\textbf{z}^{(l)}||_{2}^{2}
\label{def_loss_r1}
\end{equation}
with generalization from online with $r=1$ to mini-batch and batch learning with $r>1$ being trivial.
\par
We would like to find the solution for the optimization problem \eqref{min_loss} by adjusting the weights $w_{ij}$ and bias $b_{i}$ on each layer of the network based on individual components $\mathcal{E}$ of the loss function $\mathcal{L}$. The process of updating the weights and bias can be written as
\begin{eqnarray}
W^{(k)}          &-& \alpha \Delta W^{(k)} \\
\textbf{b}^{(k)} &-& \alpha \Delta \textbf{b}^{(k)} 
\end{eqnarray}
where $\alpha$ is some constant, often referred to as the \textit{learning rate}. 
\par
It is natural to express these weight $\Delta W^{(k)}$ and bias $\Delta \textbf{b}^{(k)}$ updates based on 
\begin{equation}
\frac{\partial \mathcal{E}}{\partial w_{ij}^{(k)}} \text{ and } \frac{\partial \mathcal{E}}{\partial b_{i}^{(k)}}
\end{equation}
that indicate the direction where the total effect of the weights and bias on the loss function component is the largest, respectively.

\begin{Lemma}
Let the feedforward neural network be defined in \eqref{forwardpropagation1} and \eqref{forwardpropagation2}, and the loss function component in \eqref{def_loss_r1}. Then, the gradient of the weights and bias can be written as 
\begin{eqnarray}
G_{w}^{(k)} &=& \textbf{v}^{(k)} \textbf{z}^{(k-1)^{T}} \phantom {111} \text{(rank-1 matrix)}  \label{gradient_expression_for_corollary} \\
G_{b}^{(k)} &=& \textbf{v}^{(k)}
\end{eqnarray}
where $G_{w}^{(k)} = [\partial \mathcal{E} / \partial w_{ij}^{(k)} ]$ is $m \times n$ matrix, $G_{b}^{(k)} =[\partial \mathcal{E} / \partial b_{i}^{(k)}]$ is $m \times 1$ vector, with 
\begin{eqnarray}
\textbf{v}^{(l)}   &=& -(\textbf{z}^{*}-\textbf{z}^{(l)}) \circ \textbf{f} \phantom{'}' (\textbf{y}^{(l)})   \phantom {111111} \label{backprop1}\\
\textbf{v}^{(k-1)} &=& ( W^{(k)^{T}} \textbf{v}^{(k)} )   \circ \textbf{f} \phantom{'}' (\textbf{y}^{(k-1)}) \phantom {111111} \label{backprop2}
\end{eqnarray}
for $k=l,...,2$, where $\circ$ is Hadamard (component-wise) product and $\textbf{f} \phantom{'}'(.)=[f'(.),...,f'(.)]^{T}$. 
\end{Lemma}

\begin{proof}
Notice that taking partial derivative of the loss function component-wise with respect to weight we can write
\begin{eqnarray}
\frac{\partial \mathcal{E}}{\partial w_{ij}^{(k)}} &=& \left( \frac{\partial \mathcal{E}}{\partial z_{i}^{(k)}} \right) \left( \frac{\partial z_{i}^{(k)}}{\partial y_{i}^{(k)}} \right) \left( \frac{\partial y_{i}^{(k)}}{\partial w_{ij}^{(k)}} \right) \label{loss_weight_total_drivative}\\
&=& \left( \frac{\partial \mathcal{E}}{\partial z_{i}^{(k)}} \right) f'(y_{i}^{(k)}) z_{j}^{(k-1)} \label{loss_weight_total_drivative2}\\
&=& v_{i}^{(k)} z_{j}^{(k-1)}
\end{eqnarray} 
where $f'$ denotes a simple ordinary derivative $df/dy$ and $v_{i}^{(k)} = \left( \frac{\partial \mathcal{E}}{\partial z_{i}^{(k)}} \right) f'(y_{i}^{(k)})$.

Also, notice that for the output layer $k=l$, using \eqref{def_loss_r1}, we have 
\begin{eqnarray}
\frac{\partial \mathcal{E}}{\partial z_{i}^{(l)}} = -(z_{i}^{*} - z_{i}^{(l)}) \label{gradient_z_output_layer}
\end{eqnarray}
while for the hidden layers, using chain rule, we have
\begin{eqnarray}
\frac{\partial \mathcal{E}}{\partial z_{i}^{(k-1)}} &=& \sum_{j=1}^{n} \left( \frac{\partial \mathcal{E}}{\partial z_{j}^{(k)}} \right) f'(y_{j}^{(k)}) w_{ji}^{(k)} \label{gradient_z_hidden_layer} \\
&=& \sum_{j=1}^{n} v_{j}^{(k)} w_{ji}^{(k)}
\end{eqnarray}
Finally, assembling the indices $i$ and $j$ into a vector and matrix form we obtain the expression for $G_{w}^{(k)}$. The derivation for $G_{b}^{(k)}$ is analogous, with  $\frac{\partial y_{i}^{(k)}}{\partial b_{i}^{(k)}}=1$ in \eqref{loss_weight_total_drivative}.  
\end{proof}

\par
Notice that the computation of the auxiliary vector $\textbf{v}^{(k)}$ in \eqref{backprop1} - \eqref{backprop2} represents the propagation of the error \eqref{def_loss_r1} from the output layer $l$ through the hidden network layers $k=l-1,...,2$. Therefore, this process is often called \textit{backward propagation}. 

\begin{Corollary}
Let the feedforward neural network be defined in \eqref{forwardpropagation1} and \eqref{forwardpropagation2}, and the loss function in \eqref{def_loss}. Then, for mini-batch of size $r$ the weight update based on the gradient $G_{w}^{(k)}$ in \eqref{gradient_expression_for_corollary} can be expressed as rank-$r$ matrix 
\begin{equation}
\Delta W^{(k)} = V^{(k)}Z^{(k-1)^{T}}
\end{equation} 
where $V^{(k)} = [\textbf{v}_{1}^{(k)}, ..., \textbf{v}_{r}^{(k)}]$ and $Z^{(k)} = [\textbf{z}_{1}^{(k)}, ...,\textbf{z}_{r}^{(k)}]$ for $r$ data pairs.
\end{Corollary}

\subsection{Hessian and Second Order Effects}
We can also incorporate second order effects based on
\begin{equation}
\frac{\partial }{\partial w_{gh}^{(k)}} \left\{ \frac{\partial \mathcal{E}}{\partial w_{ij}^{(k)}} \right\} \text{ and } 
\frac{\partial }{\partial b_{g}^{(k)}}  \left\{ \frac{\partial \mathcal{E}}{\partial b_{i}^{(k)}}  \right\}
\end{equation}
into the optimization  process for updating the weights by looking at the expression for Hessian of the neural network.

\begin{Theorem}
Let the feedforward neural network be defined in \eqref{forwardpropagation1} and \eqref{forwardpropagation2}, and the loss function component in \eqref{def_loss_r1}. Then, Hessian of weight and bias can be written as 
\begin{eqnarray}
H_{w}^{(k)} &=& \left( C^{(k)} \circ F^{(k)} + D^{(k)} \right) \otimes \left( \textbf{z}^{(k-1)} \textbf{z}^{(k-1)^{T}} \right) \phantom{111111} \label{hessian_hwk} \\
H_{b}^{(k)} &=& \phantom{2} C^{(k)} \circ F^{(k)} + D^{(k)}
\end{eqnarray}
where $H_{w}^{(k)} = [\partial / \partial w_{gh}^{(k)}  \{ \partial \mathcal{E} / \partial w_{ij}^{(k)} \}]$ is $(mn) \times (mn)$ matrix, $H_{b}^{(k)} = [ \partial / \partial b_{g}^{(k)}  \{ \partial \mathcal{E} / \partial b_{i}^{(k)} \} ]$ is $m \times m$ matrix, with 
\begin{eqnarray}
F^{(k)} &=& \textbf{f} \phantom{'}' (\textbf{y}^{(k)}) \textbf{f} \phantom{'}' (\textbf{y}^{(k)})^{T} \phantom{111111111111.}    \text{ (rank-1 matrix) } \\
D^{(k)} &=& \text{diag}(\textbf{v}^{(k)})                                                             \phantom{11111111111111111}\text{ (diagonal matrix)  } \\
C^{(l)}   &=& I \circ F^{(l)}                                                                         \phantom{1111111111111111111}\text{ (diagonal matrix) } \label{hessian_backprop1} \\ 
C^{(k-1)} &=& W^{(k)^{T}} \left( C^{(k)} \circ F^{(k)} + D^{(k)} \right) W^{(k)}      \label{hessian_backprop2}
\end{eqnarray}
where $I$ is $m \times m$ identity matrix and vectors
\begin{eqnarray}
\textbf{v}^{(l)}   &=& -(\textbf{z}^{*}-\textbf{z}^{(l)}) \circ \textbf{f} \phantom{'}'' (\textbf{y}^{(l)})   \phantom{1111111111111111111111} \label{hessian_backprop3} \\
\textbf{v}^{(k-1)} &=& ( W^{(k)^{T}} \textbf{v}^{(k)} )   \circ \textbf{f} \phantom{'}'' (\textbf{y}^{(k-1)}) \phantom{11111111111111111111}   \label{hessian_backprop4}
\end{eqnarray}
for $k=l,...,2$, where $\circ$ is Hadamard (component-wise) and $\otimes$ is Kronecker matrix product, while vectors $\textbf{f} \phantom{'}' (.)=[f'(.),...,f'(.)]^{T}$ and $\textbf{f} \phantom{'}'' (.)=[f''(.),...,f''(.)]^{T}$. 
\end{Theorem}

\begin{proof}
Notice that using \eqref{loss_weight_total_drivative2} the second derivative with respect to weight is 
\begin{align}
\frac{\partial }{\partial w_{gh}^{(k)}}  \left\{ \frac{\partial \mathcal{E}}{\partial w_{ij}^{(k)}} \right\}  
& \phantom{.}=\phantom{.} \frac{\partial }{\partial w_{gh}^{(k)}} \left\{ \left( \frac{\partial \mathcal{E}}{\partial z_{i}^{(k)}} \right) f'(y_{i}^{(k)}) z_{j}^{(k-1)}  \right\}  \\
& \phantom{.}=\phantom{.} \frac{\partial }{\partial z_{g}^{(k)}} \left\{ \left( \frac{\partial \mathcal{E}}{\partial z_{i}^{(k)}} \right) f'(y_{i}^{(k)}) \right\} f'(y_{g}^{(k)}) z_{h}^{(k-1)} z_{j}^{(k-1)} \label{hessian_weights} \\
&\phantom{.} = \phantom{.} \left[ \frac{\partial }{\partial z_{g}^{(k)}} \left\{ \frac{\partial \mathcal{E}}{\partial z_{i}^{(k)}} \right\} f'(y_{g}^{(k)}) f'(y_{i}^{(k)}) +
  \left( \frac{\partial \mathcal{E}}{\partial z_{i}^{(k)}} \right) f''(y_{i}^{(k)}) \delta_{gi} \right] z_{h}^{(k-1)} z_{j}^{(k-1)} \nonumber 
\end{align} 
where we have taken advantage of the fact that  
\begin{equation}
\frac{\partial }{\partial z_{g}^{(k)}}  \left\{  f'(y_{i}^{(k)}) \right\} f'(y_{g}^{(k)}) 
= \frac{\partial }{\partial y_{g}^{(k)}}  \left\{  f'(y_{i}^{(k)}) \right\} \\
= f''(y_{i}^{(k)}) \delta_{gi}  
\label{second_derivative_delta}
\end{equation}
and previous level output $z_{j}^{(k-1)}$ does not depend on the current level weight $w_{gh}^{(k)}$ and therefore is treated as a constant, while $\delta_{gi}$ is Kronecker delta\footnote{
Kronecker delta $\delta_{gi} = \left\{ \begin{array}{l} 1 \text{ if } g=i \\ 0 \text{ otherwise } \end{array} \right.$
}.
\par 
Let us now find an expression for the first term in \eqref{hessian_weights}. Notice that using \eqref{gradient_z_output_layer} at the output layer $k=l$ we have
\begin{equation}
\frac{\partial }{\partial z_{g}^{(l)}}  \left\{ \frac{\partial \mathcal{E}}{\partial z_{i}^{(l)}} \right\} = \frac{\partial }{\partial z_{g}^{(l)}}  \left\{ -(z_{i}^{*}-z_{i}^{(l)}) \right\} = \delta_{gi}
\label{loss_hessian_z_output_layer}
\end{equation}
while using \eqref{gradient_z_hidden_layer} at the hidden layers we may write
\begin{align}
\frac{\partial }{\partial z_{g}^{(k-1)}}  \left\{ \frac{\partial \mathcal{E}}{\partial z_{i}^{(k-1)}} \right\} & \phantom{.}=\phantom{.} \frac{\partial }{\partial z_{g}^{(k-1)}} \left\{ \sum_{j=1}^{n} \left( \frac{\partial \mathcal{E}}{\partial z_{j}^{(k)}} \right) f'(y_{j}^{(k)}) w_{ji}^{(k)} \right\} \\
&\phantom{.} = \phantom{.} \sum_{h=1}^{n} \sum_{j=1}^{n} \frac{\partial }{\partial z_{h}^{(k)}} \left\{ \left( \frac{\partial \mathcal{E}}{\partial z_{j}^{(k)}} \right) f'(y_{j}^{(k)}) \right\} f'(y_{h}^{(k)})  w_{hg}^{(k)} w_{ji}^{(k)} \label{hessian_zs} \\
&\phantom{.} = \phantom{.} \sum_{h=1}^{n} \sum_{j=1}^{n} \left[ \frac{\partial }{\partial z_{h}^{(k)}} \left\{ \frac{\partial \mathcal{E}}{\partial z_{j}^{(k)}} \right\} f'(y_{j}^{(k)}) f'(y_{h}^{(k)}) + \left( \frac{\partial \mathcal{E}}{\partial z_{j}^{(k)}} \right) f''(y_{j}^{(k)})\delta_{hj} \right] w_{hg}^{(k)} w_{ji}^{(k)} \nonumber 
\end{align} 
where we have used \eqref{second_derivative_delta} and the fact that the current level weight $w_{ji}^{(k)}$ does not depend on the previous level output $z_{g}^{(k-1)}$ and therefore is treated as a constant.
\par
We may conclude the proof by noticing the following two results. First, a matrix with block elements $C (z_{h}z_{j}) $ for $h,j=1,...,n$ can be expressed as Kronecker product $(\textbf{z}\textbf{z}^{T}) \otimes C$, which under a permutation is equivalent to $C \otimes (\textbf{z}\textbf{z}^{T})$. Second, a matrix $C = W^{T} A W$ has elements $c_{gi} = \sum_{h}\sum_{j} a_{hj}w_{ji}w_{hg}$. The former and latter results can be used to write \eqref{hessian_weights} and \eqref{hessian_zs} in the matrix form, respectively.
\par
Finally, the derivation for $H_{b}^{(k)}$ is analogous, with  $\frac{\partial y_{g}^{(k)}}{\partial b_{g}^{(k)}}=1$ in \eqref{hessian_weights}.  
\end{proof}

Notice that using \eqref{loss_hessian_z_output_layer} we may drop the double sum at level $l-1$ and write
\begin{equation}
\frac{\partial }{\partial z_{g}^{(l-1)}}  \left\{ \frac{\partial \mathcal{E}}{\partial z_{i}^{(l-1)}} \right\} 
= \sum_{j=1}^{n} \left[ \frac{\partial }{\partial z_{j}^{(l)}} \left\{ \frac{\partial \mathcal{E}}{\partial z_{j}^{(l)}} \right\} f'(y_{j}^{(l)})^{2} + \left( \frac{\partial \mathcal{E}}{\partial z_{j}^{(l)}} \right) f''(y_{j}^{(l)}) \right] w_{ji}^{(l)} w_{jg}^{(l)}  
\end{equation} 
which matches the expression obtained for a single hidden layer in \cite{Bishop1992}. However, we may not drop the double sum at an arbitrary layer $k \neq l$, because in general the term $\frac{\partial }{\partial z_{h}^{(k)}}  \left\{ \frac{\partial \mathcal{E}}{\partial z_{j}^{(k)}} \right\}$ may be nonzero even when $h \neq j$.
\par
Finally, notice that to compute Hessian we once again need to perform backward propagation for both vector $\textbf{v}^{(k)}$ in \eqref{hessian_backprop3} - \eqref{hessian_backprop4} and matrix $C^{(k)}$ in \eqref{hessian_backprop1} - \eqref{hessian_backprop2}. 

\begin{Corollary}
Suppose that we are using piecewise continuous ReLU activation function $f(y)=\max(0,y)$ in Theorem 1. Notice that its first derivative $f'(y)=1$ if $y>0$, $f'(y)=0$ if $y<0$, and $f'(y)$ is undefined if $y=0$. Also, its second derivative  $f''(y)=0$ for $\forall y \neq 0$. Then, for $\forall y_{i}^{(k)} \neq 0$ Hessian of weights can be written as
\begin{eqnarray}
H_{w}^{(k)} &=& \tilde{C}^{(k)} \otimes \left( \textbf{z}^{(k-1)} \textbf{z}^{(k-1)^{T}} \right)  \label{hessian_hwk_simplified} \\
\tilde{C}^{(k-1)} &=& W^{(k)^{T}} \tilde{C}^{(k)} W^{(k)} \circ F^{(k-1)}  \label{hessian_cktsz_simplified}   
\end{eqnarray}
with $\tilde{C}^{(l)} = I \circ F^{(l)}$ and binary matrix $F^{(k)} = \textbf{f} \phantom{'}'(\textbf{y}^{(k)}) \textbf{f} \phantom{'}'(\textbf{y}^{(k)})^{T}$ for  $k=l,...,2$. 
\end{Corollary}

\par
Notice that the eigenvalues of Kronecker product of two square $m \times m$ and $n \times n$ matrices are
\begin{equation}
\lambda_{q}(A \otimes B) = \lambda_{i}(A)\lambda_{j}(B)
\label{KroneckerEigs}
\end{equation}
for $i=1,...,m$, $j=1,...,n$ and $q=1,...,mn$, see Theorem 4.2.12 in \cite{Horn2008}. Therefore, the eigenvalues of the Hessian matrix can be expressed as
\begin{equation}
\lambda_{q}(H_{w}^{(k)}) = \left\{ 0, || \textbf{z}^{(k-1)} ||_{2}^{2} \lambda_{i}(\tilde{C}^{(k)}) \right\} 
\label{hessian_eigenvalue_set}
\end{equation} 
with $\lambda_{q}(H_{w}^{(k)})=0$ eigenvalue multiplicity being $(n-1)m$.

\section{Recurrent Neural Network (RNN)}

The recurrent neural network consists of a set of stacked fully connected layers, where neurons can receive feedback from other neurons at the previous, same and next layer at earlier time steps. However, in this paper for simplicity we will assume that the {\color{blue!80} feedback} is received only from the same level at earlier time steps, as shown in Fig. 4. 
\par
Therefore, the layers are defined by their matrix of weights $W^{(k)}$, matrix of feedback $U^{(k)}$ and vector of bias $\textbf{b}^{(k)}$. They accept an input vector $\textbf{z}^{(k-1,s)}$ from the previous level $k-1$ and the hidden state vector $\textbf{z}^{(k,s-1)}$ from the previous $s-1$ time step. They produce an output vector $\textbf{z}^{(k,s)}$ for layers $k=1,...,l$ and time steps $s=1,...,\tau$. The output of the entire neural network is often a sub-sequence of vectors $\textbf{z}^{(l,t)}$ at the last layer $l$ and time steps $t=a,...,\tau$ with starting time step $1 \le a \le \tau$.
\par
In some cases the neural network can be simplified to have sparse connections, resulting in a sparse matrix of weights $W^{(k)}$ and feedback $U^{(k)}$. Also, notice that in our example the connections within a layer have cycles due to feedback, but the connections between layers are always such that the data flow graph between them is a DAG.   

\begin{table}[h]
        \centering
        \SetVertexNormal[
                         Shape    = circle,
                 	     LineWidth= 1pt]
		\SetUpEdge[lw   = 1pt,
           	   	   color= black,
           	   	   style=->]
		\begin{tikzpicture}
   			\tikzset{VertexStyle/.append style = {minimum size = 5pt}}	
   			\Vertex[x=2, y=3, LabelOut, L=$$]{4a}    			
   			\Vertex[x=2, y=1, LabelOut, L=$$]{4c}
   			
   			\Vertex[x=7.5, y=2.8, LabelOut, L=$$]{14a}    			
   			\Vertex[x=7.5, y=1.2, LabelOut, L=$$]{14c} 
   			
   			\Vertex[x=13, y=2.5, LabelOut, L=$$]{24a}    			
   			\Vertex[x=13, y=1.5, LabelOut, L=$$]{24c} 
   			  			
   			\tikzset{VertexStyle/.append style = {minimum size = 3pt, inner sep = 0pt, color=black}}
	   		\Vertex[x=0, y=4, LabelOut, Ldist=-1.1cm, L=$x_{1}^{(*,s)}$]{1}
			\Vertex[x=0, y=0, LabelOut, Ldist=-1.1cm, L=$x_{N}^{(*,s)}$]{3}	   		
			
	   		\Vertex[x=3.5, y=3, LabelOut, Ldist=+0.1cm, L=$\ldots$]{5a}
	   		\Vertex[x=3.5, y=1, LabelOut, Ldist=+0.1cm, L=$\ldots$]{5c}      
	
			\tikzset{VertexStyle/.append style = {minimum size = 0pt}}	
			\Vertex[x=0, y=2, LabelOut, Ldist=-0.8cm, L=$\vdots$]{4b}
			\Vertex[x=1.8, y=4, LabelOut, Lpos=90, Ldist=-0.1cm, L={Input $\textbf{z}^{(0,s)} = \textbf{x}^{(*,s)}$}]{t}	
	
	   		\Edge[](1)(4a)
	   		\Edge[](3)(4a)
	   		
	   		\Edge[](1)(4c)
	   		\Edge[](3)(4c)   		
	   		
	   		\Edge[](4a)(5a)
	   		\Edge[](4c)(5c)  
	   		
	   		\draw[red,dashed, ultra thick,rounded corners] (1,0.3) rectangle (3,4);	
	   		\draw[blue!80,->,ultra thick,rounded corners] (1.75,0.25) arc (180:360:10pt);	
	   		
   			  			
   			\tikzset{VertexStyle/.append style = {minimum size = 3pt, inner sep = 0pt, color=black}}
	   		\Vertex[x=6, y=3, LabelOut, Ldist=-1.5cm, L=$z_{1}^{(k-1,s)}$]{11}
			\Vertex[x=6, y=1, LabelOut, Ldist=-1.5cm, L=$z_{n}^{(k-1,s)}$]{13}	   		
			
	   		\Vertex[x=9, y=2.8, LabelOut, Ldist=+0.2cm, L=$z_{1}^{(k,s)}$]{15a}
	   		\Vertex[x=9, y=1.2, LabelOut, Ldist=+0.2cm, L=$z_{m}^{(k,s)}$]{15c}      
	
			\tikzset{VertexStyle/.append style = {minimum size = 0pt}}	
			\Vertex[x=9.4, y=2, LabelOut, Ldist=-0.1cm, L=$\vdots$]{14b}
			\Vertex[x=5, y=2, LabelOut, Ldist=-0.4cm, L=$\vdots$]{15b}	
			\Vertex[x=7.8, y=4, LabelOut, Lpos=90, Ldist=-0.1cm, L={$\textbf{z}^{(k,s)} = \textbf{f}(W^{(k)}\textbf{z}^{(k-1,s)}+U^{(k)}\textbf{z}^{(k,s-1)}+\textbf{b}^{(k)})$}]{tt}	
	
	   		\Edge[](11)(14a)
	   		\Edge[](13)(14a)
	   		
	   		\Edge[](11)(14c)
	   		\Edge[](13)(14c)   		
	   		
	   		\Edge[](14a)(15a)
	   		\Edge[](14c)(15c)  	   		
	   		
	   		\draw[red,dashed, ultra thick,rounded corners] (8.5,0.3) rectangle (6.5,4);		   		 		
	   		\draw[blue!80,->,ultra thick,rounded corners] (7.25,0.25) arc (180:360:10pt);
	   		 			   		
   			  			
   			\tikzset{VertexStyle/.append style = {minimum size = 3pt, inner sep = 0pt, color=black}}
	   		\Vertex[x=11.5, y=2.8, LabelOut, Ldist=-0.8cm, L=$\ldots$]{21}
			\Vertex[x=11.5, y=1.2, LabelOut, Ldist=-0.8cm, L=$\ldots$]{23}	   		
			
	   		\Vertex[x=14.5, y=2.5, LabelOut, Ldist=+0.1cm, L=$z_{1}^{(l,t)}$]{25a}
	   		\Vertex[x=14.5, y=1.5, LabelOut, Ldist=+0.1cm, L=$z_{M}^{(l,t)}$]{25c}      
	
			\tikzset{VertexStyle/.append style = {minimum size = 0pt}}	
			\Vertex[x=14.9, y=2, LabelOut, Ldist=-0.3cm, L=$\vdots$]{24b}
			\Vertex[x=13.2, y=4, LabelOut, Lpos=90, Ldist=-0.1cm, L={Output $\textbf{z}^{(l,t)}$}]{ttt}	
			
				
	   		\Edge[](21)(24a)
	   		\Edge[](23)(24a)
	   		
	   		\Edge[](21)(24c)
	   		\Edge[](23)(24c)   		
	   		
	   		\Edge[](24a)(25a)
	   		\Edge[](24c)(25c)  	   		
	   		
	   		\draw[red,dashed, ultra thick,rounded corners] (14,0.3) rectangle (12,4);		   		 			   		
	   		\draw[blue!80,->,ultra thick,rounded corners] (12.75,0.25) arc (180:360:10pt);			   		
	   		 			   		
		\end{tikzpicture} 
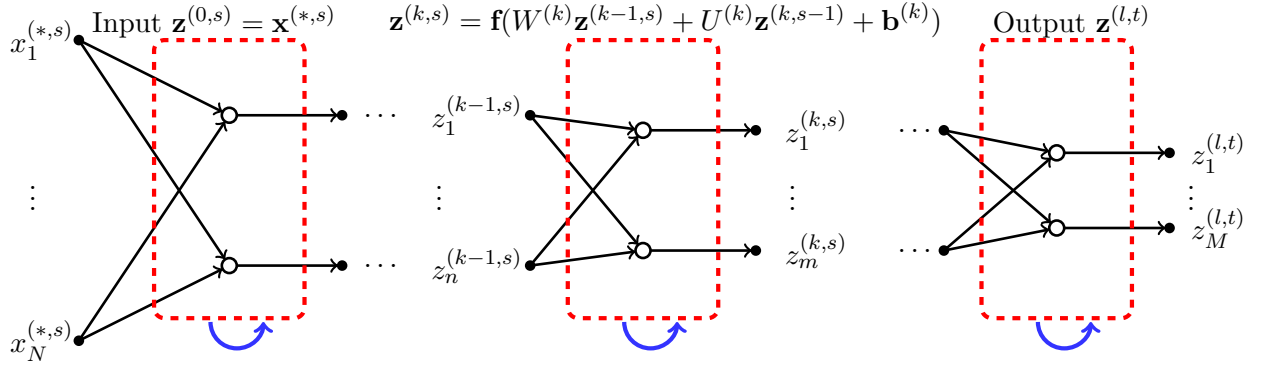
\caption*{Fig. 4: A Sample Recurrent Neural Network (RNN) with $k=1,...,l$ Levels and $s=1,...,\tau$ Time Steps}    
\end{table}
\setcounter{table}{0} 
\setcounter{figure}{4} 

\subsection{Forward Propagation}
Let us assume that we are given an input sequence $\textbf{x}^{(*,s)}$, then we can compute an output sequence $\textbf{z}^{(l,s)}$ generated by the neural network by repeated applications of the formula
\begin{eqnarray}
\textbf{z}^{(k,s)} &=& \textbf{f}(\textbf{y}^{(k,s)})                                            \label{forwardpropagation1_rnn} \\
\textbf{y}^{(k,s)} &=& W^{(k)}\textbf{z}^{(k-1,s)}+U^{(k)}\textbf{z}^{(k,s-1)}+\textbf{b}^{(k)}  \label{forwardpropagation2_rnn}
\end{eqnarray}
for $k=1,...,l$ and $s=1,...,\tau$, where initial hidden state $\textbf{z}^{(k,0)}=\textbf{0}$ and input $\textbf{z}^{(0,s)} = \textbf{x}^{(*,s)}$. Notice that the final output is often a sub-sequence $\textbf{z}^{(l,t)}$, where $t=a,...,\tau$ with starting time step $1 \le a \le \tau$. This process is called \textit{forward propagation}.

\addtolength{\textheight}{+0.5cm} 

\subsection{Backward Propagation (Through Time)}
Let us assume that we have a data sample $(\textbf{x}^{*},\textbf{z}^{*})$ from the training data set $\mathcal{D}$, as we have discussed in the introduction. Notice that here the input $\textbf{x}^{*}$ and output $\textbf{z}^{*}$ are actually a sequence $\textbf{x}^{(*,s)}$ and $\textbf{z}^{(*,t)}$ for time steps $s=1,...,\tau$ and $t=a,...,\tau$ with starting time step $1 \le a \le \tau$, respectively. Also, notice that using forward propagation we may compute the actual output $\textbf{z}^{(l,t)}$ and the error (associated with scaled loss $\frac{1}{2}\mathcal{L}$)  
\begin{equation}
\mathcal{E} = \sum_{t=a}^{\tau} \mathcal{E}^{(t)} = \sum_{t=a}^{\tau} \frac{1}{2}||\textbf{z}^{(*,t)}-\textbf{z}^{(l,t)}||_{2}^{2}
\label{def_loss_r1_rnn}
\end{equation}
with generalization from online with $r=1$ to mini-batch and batch learning with $r>1$ being trivial.
\par
We would like to find the solution for the optimization problem \eqref{min_loss} by adjusting the weights $w_{ij}$, feedback $u_{ij}$ and bias $b_{i}$ on each layer of the network based on individual components $\mathcal{E}$ of the loss function $\mathcal{L}$. So that the updating process can be written as
\begin{eqnarray}
W^{(k)}          &-& \alpha \Delta W^{(k)} \\
U^{(k)}          &-& \alpha \Delta U^{(k)} \\
\textbf{b}^{(k)} &-& \alpha \Delta \textbf{b}^{(k)} 
\end{eqnarray}
where $\alpha$ is some constant, often referred to as the \textit{learning rate}. 
\par
It is natural to express these weight $\Delta W^{(k)}$, feedback $\Delta U^{(k)}$ and bias $\Delta \textbf{b}^{(k)}$ updates based on 
\begin{equation}
\frac{\partial \mathcal{E}}{\partial w_{ij}^{(k)}} \text{ , }   
\frac{\partial \mathcal{E}}{\partial u_{ij}^{(k)}} \text{ and } 
\frac{\partial \mathcal{E}}{\partial b_{i}^{(k)}}               
\end{equation}
that indicate the direction where the total effect of the weights, feedback and bias on the loss function component is the largest, respectively. Notice in turn  that these quantities can be expressed through the sum of their sub-components 
\begin{equation}
\frac{\partial \mathcal{E}^{(t)}}{\partial w_{ij}^{(k)}} \text{ , } \frac{\partial \mathcal{E}^{(t)}}{\partial u_{ij}^{(k)}} \text{ and } \frac{\partial \mathcal{E}^{(t)}}{\partial b_{i}^{(k)}}
\end{equation}
for $t=a,...,\tau$, which will be our focus next.

\begin{Lemma}
Let the recurrent neural network be defined in \eqref{forwardpropagation1_rnn} and \eqref{forwardpropagation2_rnn}, and the loss function components in \eqref{def_loss_r1_rnn}. Then, the gradient of the weights and bias can be written as  
\begin{eqnarray}
G_{w}^{(k,t)} &=& \sum_{s=1}^{t} \tilde{\textbf{v}}^{(k,t,s)} \textbf{z}^{(k-1,s)^{T}} \phantom{111} \text{(rank-t matrix)} \label{gradient_expression_for_corollary_rnn} \\
G_{u}^{(k,t)} &=& \sum_{s=1}^{t} \tilde{\textbf{v}}^{(k,t,s)} \textbf{z}^{(k,s-1)^{T}} \phantom{111} \text{(rank-t matrix)} \label{gradient_expression_for_corollary2_rnn} \\
G_{b}^{(k,t)} &=& \sum_{s=1}^{t} \tilde{\textbf{v}}^{(k,t,s)}
\end{eqnarray}
where $G_{w}^{(k,t)} = [\partial \mathcal{E}^{(t)} / \partial w_{ij}^{(k)}]$ is $m \times n$ matrix, $G_{u}^{(k,t)}  = [\partial \mathcal{E}^{(t)} / \partial u_{ij}^{(k)}]$ is $m \times m$ matrix and  $G_{b}^{(k,t)}=[\partial \mathcal{E}^{(t)} / \partial b_{i}^{(k)}]$ is $m \times 1$ vector, with  
\begin{equation}
\tilde{\textbf{v}}^{(k,t,s)} = \left( \prod_{h=s}^{t-1} U^{(k)} \text{diag}(\textbf{f} \phantom{'}' (\textbf{y}^{(k,h)}) ) \right)^{T} \textbf{v}^{(k,t)}  \phantom{11111111111}
\end{equation}
and
\begin{eqnarray}
\textbf{v}^{(l,t)}   &=& - \left(   \textbf{z}^{(*,t)}-\textbf{z}^{(l,t)}                                                               \right) \circ \textbf{f} \phantom{'}' (\textbf{y}^{(l,t)})   \label{backprop1_rnn}\\
\textbf{v}^{(k-1,t)} &=&   \left( W^{(k)^{T}} \textbf{v}^{(k,t)} + U^{(k-1)^{T}} \textbf{v}^{(k-1,t+1)} \right) \circ \textbf{f} \phantom{'}' (\textbf{y}^{(k-1,t)}) \label{backprop3_rnn}
\end{eqnarray}
for $k=l,...,2$ and $t=\tau,...,a$, where we consider the terms for time $t+1>\tau$ to be zero. Also, $\circ$ is Hadamard (component-wise) product, $\textbf{e}=[1,...,1]^{T}$ and $\textbf{f} \phantom{'}' (.)=[f'(.),...,f'(.)]^{T}$.  
\end{Lemma}

\addtolength{\textheight}{-0.5cm} 
\clearpage
\newpage

\begin{proof}
Notice that in RNNs all components of output vector $\textbf{z}^{(k,t)}$ depend on weight $w_{ij}^{(k)}$ due to presence of feedback matrix $U$, unlike FNNs where only $i$-th component depends on this weight. Therefore, taking partial derivative of the loss function component-wise with respect to weight we can write
\begin{eqnarray}
\frac{\partial \mathcal{E}^{(t)}}{\partial w_{ij}^{(k)}} 
&=& \sum_{g=1}^{m} \left( \frac{\partial \mathcal{E}^{(t)}}{\partial z_{g}^{(k,t)}} \right) 
                   \left( \frac{\partial z_{g}^{(k,t)}}{\partial y_{g}^{(k,t)}}         \right) 
                   \left( \frac{\partial y_{g}^{(k,t)}}{\partial w_{ij}^{(k)}}          \right) 
                   \label{loss_weight_total_drivative_rnn} \\
&=& \sum_{g=1}^{m} \left( \frac{\partial \mathcal{E}^{(t)}}{\partial z_{g}^{(k,t)}} \right) 
                   f'(y_{g}^{(k,t)}) 
                   \left( z_{j}^{(k-1,t)} \delta_{gi} + \sum_{h=1}^{m} u_{gh}^{(k)} \frac{\partial z_{h}^{(k,t-1)}}{\partial w_{ij}^{(k)}} \right) 
                   \label{loss_weight_total_drivative2_rnn} \\
&=& \sum_{g=1}^{m} v_{g}^{(k,t)} 
                   \left( z_{j}^{(k-1,t)} \delta_{gi} + \sum_{h=1}^{m} \tilde{u}_{gh}^{(k,t-1)} \frac{\partial y_{h}^{(k,t-1)}}{\partial w_{ij}^{(k)}} \right)
\end{eqnarray} 
where $v_{i}^{(k,t)} = \left( \frac{\partial \mathcal{E}^{(t)}}{\partial z_{i}^{(k,t)}} \right) f'(y_{i}^{(k,t)})$, $\tilde{u}_{gh}^{(k,t)} = u_{gh}^{(k)} f'(y_{h}^{(k,t)})$, $\frac{\partial y_{g}^{(k,1)}}{\partial w_{ij}^{(k)}}=z_{j}^{(k-1,1)}\delta_{ig}$ and $f'$ denotes a simple ordinary derivative $df/dy$, while $\delta_{gi}$ is the Kronecker delta.
\par
We can unroll the above expression for a few time steps to obtain
\begin{eqnarray}
\frac{\partial \mathcal{E}^{(t)}}{\partial w_{ij}^{(k)}} 
&=& v_{i}^{(k,t)} z_{j}^{(k-1,t)} + \sum_{g=1}^{m} v_{g}^{(k,t)} \tilde{u}_{gi}^{(k,t-1)} z_{j}^{(k-1,t-1)} + \\
&+& \sum_{g=1}^{m} v_{g}^{(k,t)} \sum_{h=1}^{m} \tilde{u}_{gh}^{(k,t-1)}   \tilde{u}_{hi}^{(k,t-2)} z_{j}^{(k-1,t-2)} + ... \nonumber                    
\label{unrolled_gradient_expression}
\end{eqnarray}  
where the last sum disappears for the term involving $z_{j}^{(k-1,s)}$ due to Kronecker delta. Therefore, in matrix form leading to creation of the terms
\begin{align}
& \left( I z_{j}^{(k-1,t)} + \tilde{U}^{(k,t-1)} z_{j}^{(k-1,t-1)} + \tilde{U}^{(k,t-1)} \tilde{U}^{(k,t-2)} z_{j}^{(k-1,t-2)} + ... \right)^{T} \textbf{v}^{(k,t)} \nonumber \\     
&= \sum_{s=1}^{t} \left( \prod_{h=s}^{t-1} \tilde{U}^{(k,h)} \right)^{T} z_{j}^{(k-1,s)} \textbf{v}^{(k,t)} 
\label{unrolled_hessian_expression_matrix_form_rnn}
\end{align}
\par
Also, notice that for the output layer $k = l$ at time $s=t$, using using \eqref{def_loss_r1_rnn}, we have 
\begin{eqnarray}
\frac{\partial \mathcal{E}^{(t)}}{\partial z_{i}^{(l,t)}} 
&=& -(z_{i}^{(*,t)} - z_{i}^{(l,t)})
\label{gradient_z_output_layer_rnn1}
\end{eqnarray}
while for other layers and time steps, using chain rule, we have
\begin{align}
\frac{\partial \mathcal{E}^{(t)}}{\partial z_{i}^{(k-1,t)}} 
&= \sum_{j=1}^{n} \left( \frac{\partial \mathcal{E}^{(t)}}{\partial z_{j}^{(k,t)}} \right) f'(y_{j}^{(k,t)})  w_{ji}^{(k)} 
 + \sum_{j=1}^{m} \left( \frac{\partial \mathcal{E}^{(t)}}{\partial z_{j}^{(k-1,t+1)}} \right) f'(y_{j}^{(k-1,t+1)})  u_{ji}^{(k-1)} \label{gradient_z_hidden_layer_rnn} \\
&= \sum_{j=1}^{n} v_{j}^{(k,t)} w_{ji}^{(k)} + \sum_{j=1}^{m} v_{j}^{(k-1,t+1)} u_{ji}^{(k-1)}
\end{align}
\par
Finally, assembling the indices $i$ and $j$ into a vector and matrix forms we obtain an expression for $G_{w}^{(k,t)}$. The derivation for $G_{u}^{(k,t)}$ and $G_{b}^{(k,t)}$ is analogous, with exception that 
\begin{equation}
\frac{\partial y_{g}^{(k,t)}}{\partial u_{ij}^{(k)}}=z_{j}^{(k,t-1)}\delta_{gi} + \sum_{h=1}^{m} u_{gh}^{(k)} \frac{\partial z_{h}^{(k,t-1)}}{\partial u_{ij}^{(k)}}
\end{equation}
and
\begin{equation}
\frac{\partial y_{g}^{(k,t)}}{\partial b_{i}^{(k)}}=\delta_{gi} \phantom{11} + \phantom{111} \sum_{h=1}^{m} u_{gh}^{(k)} \frac{\partial z_{h}^{(k,t-1)}}{\partial b_{i}^{(k)}}
\end{equation}
in \eqref{loss_weight_total_drivative_rnn}, respectively.  
\end{proof}

\par
Notice that the computation of the auxiliary vector $\textbf{v}^{(k,t)}$ in \eqref{backprop1_rnn} - \eqref{backprop3_rnn} represents the propagation of the error \eqref{def_loss_r1_rnn} from the output layer $l$ through the hidden network layers $k=l-1,...,2$ and time steps $t=\tau-1,...,a$. Therefore, this process is often called \textit{backward propagation through time} (BPTT).

\begin{Corollary}
Let the recurrent neural network be defined in \eqref{forwardpropagation1_rnn} and \eqref{forwardpropagation2_rnn}, and the loss function in \eqref{def_loss}. Then for mini-batch of size $r$ the weight update based on the gradient $G_{w}^{(k,t)}$ in \eqref{gradient_expression_for_corollary_rnn} and $G_{u}^{(k,t)}$ in \eqref{gradient_expression_for_corollary2_rnn} can be expressed as rank-$rt$ matrix 
\begin{eqnarray}
\Delta W^{(k)} &=& \sum_{s=1}^{t} \tilde{V}^{(k,s)}Z^{(k-1,s)^{T}} \\
\Delta U^{(k)} &=& \sum_{s=1}^{t} \tilde{V}^{(k,s)}Z^{(k,s-1)^{T}}
\end{eqnarray} 
where $\tilde{V}^{(k,s)} = [\tilde{\textbf{v}}_{1}^{(k,s)}, ..., \tilde{\textbf{v}}_{r}^{(k,s)}]$ and $Z^{(k,s)} = [\textbf{z}_{1}^{(k,s)}, ...,\textbf{z}_{r}^{(k,s)}]$ for $r$ data pairs.
\end{Corollary}

\subsection{Hessian and Second Order Effects (Through Time)}
We can also incorporate second order effects based on
\begin{equation}
\frac{\partial }{\partial w_{pq}^{(k)}} \left\{ \frac{\partial \mathcal{E}^{(t)}}{\partial w_{ij}^{(k)}} \right\} \text{ , }
\frac{\partial }{\partial u_{pq}^{(k)}} \left\{ \frac{\partial \mathcal{E}^{(t)}}{\partial u_{ij}^{(k)}} \right\} \text{ and } 
\frac{\partial }{\partial b_{p}^{(k)}}  \left\{ \frac{\partial \mathcal{E}^{(t)}}{\partial b_{i}^{(k)}}  \right\}
\end{equation}
into the optimization  process for updating the weights by looking at the expression for Hessian of the neural network.

\clearpage
\newpage

\begin{Theorem}
Let the recurrent neural network be defined in \eqref{forwardpropagation1_rnn} and \eqref{forwardpropagation2_rnn}, and the loss function component in \eqref{def_loss_r1_rnn}. Then, Hessian of weight and bias can be written as 
\begin{eqnarray}
H_{w}^{(k,t)} &=& \sum_{s=1}^{t} \sum_{\zeta =1}^{t} \tilde{C}^{(k,t,s,\zeta)} \otimes \left( \textbf{z}^{(k-1,s)} \textbf{z}^{(k-1,\zeta )^{T}} \right) \phantom{111111111}\\
H_{u}^{(k,t)} &=& \sum_{s=1}^{t} \sum_{\zeta =1}^{t} \tilde{C}^{(k,t,s,\zeta)} \otimes \left( \textbf{z}^{(k,s-1)} \textbf{z}^{(k,\zeta -1)^{T}} \right) \phantom{111111111} \\
H_{b}^{(k,t)} &=& \sum_{s=1}^{t} \sum_{\zeta =1}^{t} \tilde{C}^{(k,t,s,\zeta)} \phantom{111111111}
\end{eqnarray}
where $H_{w}^{(k,t)} = [\partial / \partial w_{pq}^{(k)}  \{ \partial \mathcal{E}^{(t)} / \partial w_{ij}^{(k)} \}]$ is $(mn) \times (mn)$ matrix, $H_{u}^{(k,t)} = [\partial / \partial u_{pq}^{(k)}  \{ \partial \mathcal{E}^{(t)} / \partial u_{ij}^{(k)} \}]$ is $(mm) \times (mm)$ matrix, $H_{b}^{(k,t)} = [ \partial / \partial b_{p}^{(k)}  \{ \partial \mathcal{E}^{(t)} / \partial b_{i}^{(k)} \} ]$ is $m \times m$ matrix , with 
\begin{equation}
\tilde{C}^{(k,t,s,\zeta)} = \left( \prod_{h=s}^{t-1} U^{(k)} \text{diag} ( \textbf{f} \phantom{'}' (\textbf{y}^{(k,h)})) \right)^{T} A^{(k,t)} \left( \prod_{h=\zeta}^{t-1} U^{(k)} \text{diag} ( \textbf{f} \phantom{'}' (\textbf{y}^{(k,h)})) \right) 
\label{hessian_cktsz_rnn}
\end{equation}
and
\begin{eqnarray}
A^{(k,t)}   &=& C^{(k,t)} \circ F^{(k,t)} + D^{(k,t)}                                                                         \label{hessian_aks_rnn} \\
F^{(k,t)}   &=& \textbf{f} \phantom{'}' (\textbf{y}^{(k,t)}) \textbf{f} \phantom{'}' (\textbf{y}^{(k,t)})^{T}  \phantom{11111111111.}       \text{ (rank-1 matrix) }   \\ 
D^{(k,t)}   &=& \text{diag}(\textbf{v}^{(k,t)})                                                                \phantom{11111111111111111}  \text{ (diagonal matrix) } \\
C^{(l,t)}   &=& I \circ F^{(l,t)}                                                                                              \label{hessian_backprop1_rnn} \\ 
C^{(k-1,t)} &=& W^{(k)^{T}} A^{(k,t)} W^{(k)} + U^{(k-1)^{T}} A^{(k-1,t+1)} U^{(k-1)}                          \phantom{1}     \label{hessian_backprop3_rnn}
\end{eqnarray}
where $I$ is $m \times m$ identity matrix and vectors
\begin{eqnarray}
\textbf{v}^{(l,t)}     &=& - \left(    \textbf{z}^{(*,t)}-\textbf{z}^{(l,t)}                                                               \right)    \circ \textbf{f} \phantom{'}'' (\textbf{y}^{(l,t)})       
\label{hessian_backprop4_rnn}\\
\textbf{v}^{(k-1,t)} &=& \left( W^{(k)^{T}} \textbf{v}^{(k,t)} + U^{(k-1)^{T}} \textbf{v}^{(k-1,t+1)} \right)   \circ \textbf{f} \phantom{'}'' (\textbf{y}^{(k-1,t)})  \label{hessian_backprop6_rnn}
\end{eqnarray}
for $k=l,...,2$ and $t=\tau,...,a$, where we consider the terms for time $t+1>\tau$ to be zero. Also, $\circ$ is Hadamard (component-wise) and $\otimes$ is Kronecker matrix product, while vectors $\textbf{f} \phantom{'}' (.)=[f'(.),...,f'(.)]^{T}$ and $\textbf{f} \phantom{'}'' (.)=[f''(.),...,f''(.)]^{T}$. 
\end{Theorem}

\addtolength{\textheight}{+0.5cm} 

\begin{proof}
Notice that using \eqref{loss_weight_total_drivative2_rnn} the second derivative with respect to weight is 
\begin{align}
\frac{\partial }{\partial w_{pq}^{(k)}}  \left\{ \frac{\partial \mathcal{E}^{(t)}}{\partial w_{ij}^{(k)}} \right\} 
& =
\sum_{r=1}^{m} 
\frac{\partial }{\partial z_{r}^{(k,t)}} \left\{  
                                             \sum_{g=1}^{m} \left( \frac{\partial \mathcal{E}^{(t)}}{\partial z_{g}^{(k,t)}} \right) 
                                             f'(y_{g}^{(k,t)})
                                             \left( z_{j}^{(k-1,t)} \delta_{gi} + \sum_{h=1}^{m} u_{gh}^{(k)} \frac{\partial z_{h}^{(k,t-1)}}{\partial w_{ij}^{(k)}} \right)
                                             \right\} \nonumber \\
& \phantom{111111111111111111111111} \times     
f'(y_{r}^{(k,t)}) 
\left( z_{q}^{(k-1,t)} \delta_{rp} + \sum_{h=1}^{m} u_{rh}^{(k)} \frac{\partial z_{h}^{(k,t-1)}}{\partial w_{pq}^{(k)}} \right) \\
& =
\sum_{r=1}^{m} \sum_{g=1}^{m}  \left[ \frac{\partial }{\partial z_{r}^{(k,t)}} \left\{ \frac{\partial \mathcal{E}^{(t)}}{\partial z_{g}^{(k,t)}} \right\} 
                                      f'(y_{g}^{(k,t)}) f'(y_{r}^{(k,t)})
                                      +
                                      \left( \frac{\partial \mathcal{E}^{(t)}}{\partial z_{g}^{(k,t)}} \right) f''(y_{g}^{(k,t)}) \delta_{rg}
                               \right] \nonumber \\
& \phantom{111} \times 
\left( z_{j}^{(k-1,t)} \delta_{gi} + \sum_{h=1}^{m} u_{gh}^{(k)} \frac{\partial z_{h}^{(k,t-1)}}{\partial w_{ij}^{(k)}} \right)                                                                    
\left( z_{q}^{(k-1,t)} \delta_{rp} + \sum_{h=1}^{m} u_{rh}^{(k)} \frac{\partial z_{h}^{(k,t-1)}}{\partial w_{pq}^{(k)}} \right) 
\label{hessian_weights_rnn}
\end{align} 
where we have taken advantage of the fact that  
\begin{eqnarray}
  \frac{\partial }{\partial z_{r}^{(k,t)}}  \left\{  f'(y_{g}^{(k,t)}) \right\} f'(y_{r}^{(k,t)})    
= \frac{\partial }{\partial y_{r}^{(k,t)}}  \left\{  f'(y_{g}^{(k,t)}) \right\} 
= f''(y_{g}^{(k,t)}) \delta_{rg} 
\label{second_derivative_delta_rnn}
\end{eqnarray}
and neither previous level output $z_{j}^{(k-1,t)}$ nor previous time step quantity $\partial z_{h}^{(k,t-1)} / \partial w_{ij}^{(k)}$ depend on the current level output $z_{r}^{(k,t)}$ and therefore are treated as a constants. 
\par
Letting $a_{gr}^{(k,t)}$ be the term in square brackets in \eqref{hessian_weights_rnn}, we can unroll this expression as
\begin{align}
\sum_{r=1}^{m} \sum_{g=1}^{m} a_{gr}^{(k,t)} 
&\times \left( z_{j}^{(k-1,t)} \delta_{gi} +  \tilde{u}_{gi}^{(k,t-1)} z_{j}^{(k-1,t-1)} + \sum_{h=1}^{m} \tilde{u}_{gh}^{(k,t-1)} \tilde{u}_{hi}^{(k,t-2)} z_{j}^{(k-1,t-2)} + ... \right) \nonumber \\                                                                    
&\times \left( z_{q}^{(k-1,t)} \delta_{rp} +  \tilde{u}_{rp}^{(k,t-1)} z_{q}^{(k-1,t-1)} + \sum_{h=1}^{m} \tilde{u}_{rh}^{(k,t-1)} \tilde{u}_{hp}^{(k,t-2)} z_{q}^{(k-1,t-2)} + ... \right) 
\label{unrolled_hessian_expression_rnn}
\end{align}
where $\tilde{u}_{gh}^{(k,t)} = u_{gh}^{(k)} f'(y_{h}^{(k,t)})$ and the last sum disappears for the term involving $z_{j}^{(k-1,s)}$ due to Kronecker delta. Therefore, in matrix form leading to creation of the terms
\begin{align}
& \left( I z_{j}^{(k-1,t)} + \tilde{U}^{(k,t-1)} z_{j}^{(k-1,t-1)} + \tilde{U}^{(k,t-1)} \tilde{U}^{(k,t-2)} z_{j}^{(k-1,t-2)} + ... \right)^{T} \nonumber \\ 
& \phantom{111111111111} \times  A^{(k,t)} \times 
\left( I z_{q}^{(k-1,t)} + \tilde{U}^{(k,t-1)} z_{q}^{(k-1,t-1)} + \tilde{U}^{(k,t-1)} \tilde{U}^{(k,t-2)} z_{q}^{(k-1,t-2)} + ... \right) = \nonumber \\          
&= \sum_{s=1}^{t} \left( \prod_{h=s}^{t-1} \tilde{U}^{(k,h)} \right)^{T} z_{j}^{(k-1,s)}  \times A^{(k,t)} \times  \sum_{\zeta =1}^{t} \left( \prod_{h=\zeta}^{t-1} \tilde{U}^{(k,h)} \right) z_{q}^{(k-1,\zeta)} 
\label{unrolled_hessian_expression_matrix_form_rnn}
\end{align}
where we have used the fact that $c_{ip} = \sum_{g}\sum_{r} a_{gr}u_{rp}u_{gi}$ are elements of matrix $C = U^{T} A U$.
\par 
Let us now find an expression for the first term in \eqref{hessian_weights_rnn}. Notice that using \eqref{gradient_z_output_layer_rnn1} at the output layer $k=l$ and time $s=t$ we have
\begin{equation}
\frac{\partial }{\partial z_{r}^{(l,t)}}  \left\{ \frac{\partial \mathcal{E}^{(t)}}{\partial z_{g}^{(l,t)}} \right\} = \frac{\partial }{\partial z_{r}^{(l,t)}}  \left\{ -(z_{g}^{(*,t)} - z_{g}^{(l,t)}) \right\} = \delta_{rg}
\label{loss_hessian_z_output_layer_rnn1} 
\end{equation}
while using \eqref{gradient_z_hidden_layer_rnn} and \eqref{second_derivative_delta_rnn} for other layers and time steps we may write
\begin{align}
&\frac{\partial }{\partial z_{r}^{(k-1,t)}} \left\{ \frac{\partial \mathcal{E}^{(t)}}{\partial z_{g}^{(k-1,t)}} \right\} = \nonumber \\
&= \sum_{h=1}^{n} \sum_{\zeta =1}^{n} \frac{\partial }{\partial z_{h}^{(k,t)}} \left\{ \left( \frac{\partial \mathcal{E}^{(t)}}{\partial z_{\zeta }^{(k,t)}} \right) f'(y_{\zeta }^{(k,t)})  w_{\zeta g}^{(k)} \right\}
    f'(y_{h}^{(k,t)}) w_{hr}^{(k)}  \phantom{1} + \nonumber \\     
&\phantom{=} \sum_{h=1}^{m} \sum_{\zeta =1}^{m} \frac{\partial }{\partial z_{h}^{(k-1,t+1)}} \left\{ \left( \frac{\partial \mathcal{E}^{(t)}}{\partial z_{\zeta }^{(k-1,t+1)}} \right) f'(y_{\zeta }^{(k-1,t+1)})  u_{\zeta g}^{(k-1)} \right\} 
    f'(y_{h}^{(k-1,t+1)}) u_{hr}^{(k-1)} \label{hessian_zs_rnn1} \\
&= \sum_{h=1}^{n} \sum_{\zeta =1}^{n} \left[ \frac{\partial }{\partial z_{h}^{(k,t)}} \left\{ \frac{\partial \mathcal{E}^{(t)}}{\partial z_{\zeta }^{(k,t)}} \right\} f'(y_{\zeta }^{(k,t)}) f'(y_{h}^{(k,t)}) 
                                          + \left( \frac{\partial \mathcal{E}^{(t)}}{\partial z_{\zeta }^{(k,t)}} \right) f''(y_{\zeta }^{(k,t)}) \delta_{h\zeta } \right] w_{\zeta g}^{(k)} w_{hr}^{(k)}  \phantom{1} + \nonumber \\    
&\phantom{=} \sum_{h=1}^{m} \sum_{\zeta =1}^{m} \left[ \frac{\partial }{\partial z_{h}^{(k-1,t+1)}} \left\{ \frac{\partial \mathcal{E}^{(t)}}{\partial z_{\zeta }^{(k-1,t+1)}} \right\} f'(y_{\zeta }^{(k-1,t+1)}) f'(y_{h}^{(k-1,t+1)}) \right. \nonumber \\ &\phantom{1111111111111111111111111111111111111} \left.
                                          + \left( \frac{\partial \mathcal{E}^{(t)}}{\partial z_{\zeta }^{(k-1,t+1)}} \right) f''(y_{\zeta }^{(k-1,t+1)}) \delta_{h\zeta }  \right] u_{\zeta g}^{(k-1)} u_{hr}^{(k-1)} \nonumber  
\end{align}
\par
Notice that in \eqref{hessian_zs_rnn1} the terms
\begin{eqnarray}
\frac{\partial }{\partial z_{r}^{(k,t)}} \left\{ \sum_{\zeta =1}^{m} \left( \frac{\partial \mathcal{E}^{(t)}}{\partial z_{\zeta }^{(k-1,t+1)}} \right) f'(y_{\zeta }^{(k-1,t+1)}) u_{\zeta g}^{(k-1)} \right\} = 0 \label{target_hessian_zero1_rnn} \\ 
\frac{\partial }{\partial z_{r}^{(k-1,t+1)}} \left\{ \sum_{\zeta =1}^{n} \left( \frac{\partial \mathcal{E}^{(t)}}{\partial z_{\zeta }^{(k,t)}} \right) f'(y_{\zeta }^{(k,t)}) w_{\zeta g}^{(k)\phantom{-1}}   \right\} = 0 \label{target_hessian_zero2_rnn} 
\end{eqnarray}
because in \eqref{target_hessian_zero1_rnn} functions on the previous level $k-1$ do not depend on the current level $k$ and in \eqref{target_hessian_zero2_rnn} functions at the previous time $t$ do not depend on the current time $t+1$.
\par

We may conclude the proof by noticing the following two results. First, a matrix with block elements $C (z_{q}z_{j}) $ for $q,j=1,...,n$ can be expressed as Kronecker product $(\textbf{z}\textbf{z}^{T}) \otimes C$, which under a permutation is equivalent to $C \otimes (\textbf{z}\textbf{z}^{T})$. Second, a matrix $C = W^{T} A W$ has elements $c_{rg} = \sum_{h}\sum_{\zeta} a_{h\zeta}w_{\zeta g}w_{hr}$. The former and latter results can be used to write \eqref{unrolled_hessian_expression_matrix_form_rnn} and \eqref{hessian_zs_rnn1} in the matrix form, respectively. 
\par
Finally, the derivation for $H_{u}^{(k)}$ and $H_{b}^{(k)}$ is analogous, with exception that 
\begin{equation}
\frac{\partial y_{r}^{(k,t)}}{\partial u_{pq}^{(k)}}= z_{q}^{(k,t-1)}  \delta_{rp} + \sum_{h=1}^{m} u_{rh}^{(k)} \frac{\partial z_{h}^{(k,t-1)}}{\partial u_{pq}^{(k)}}
\end{equation}
and 
\begin{equation}
\frac{\partial y_{r}^{(k,t)}}{\partial b_{p}^{(k)}} = \delta_{rp} \phantom{11} + \phantom{111} \sum_{h=1}^{m} u_{rh}^{(k)} \frac{\partial z_{h}^{(k,t-1)}}{\partial b_{p}^{(k)}}
\end{equation}
in \eqref{hessian_weights_rnn}, respectively.
\end{proof}

\par
Finally, notice that to compute Hessian we once again need to perform backward propagation through time for both vector $\textbf{v}^{(k,t)}$ in \eqref{hessian_backprop4_rnn} - \eqref{hessian_backprop6_rnn} and matrix $C^{(k,t)}$ in \eqref{hessian_backprop1_rnn} - \eqref{hessian_backprop3_rnn}. 

\begin{Corollary}
Suppose that we are using piecewise continuous ReLU activation function $f(y)=\max(0,y)$ in Theorem 2. Notice that its first derivative $f'(y)=1$ if $y>0$, $f'(y)=0$ if $y<0$, and $f'(y)$ is undefined if $y=0$. Also, its second derivative  $f''(y)=0$ for $\forall y \neq 0$. Then, for $\forall y_{i}^{(k,s)} \neq 0$ Hessian of weights can be written as
\begin{eqnarray}
H_{w}^{(k,t)} &=& \sum_{s=1}^{t} \sum_{\zeta =1}^{t} \tilde{C}^{(k,t,s,\zeta)} \otimes \left( \textbf{z}^{(k-1,s)} \textbf{z}^{(k-1,\zeta )^{T}} \right) \\
H_{u}^{(k,t)} &=& \sum_{s=1}^{t} \sum_{\zeta =1}^{t} \tilde{C}^{(k,t,s,\zeta)} \otimes \left( \textbf{z}^{(k,s-1)} \textbf{z}^{(k,\zeta -1)^{T}} \right) \\
\tilde{C}^{(k,t,s,\zeta)} &=& \left( \prod_{h=s}^{t-1} U^{(k)} \text{diag} ( \textbf{f} \phantom{'}' (\textbf{y}^{(k,h)})) \right)^{T} A^{(k,t)} \left( \prod_{h=\zeta}^{t-1} U^{(k)} \text{diag} ( \textbf{f} \phantom{'}' (\textbf{y}^{(k,h)})) \right) \label{hessian_cktsz_rnn_simplified} 
\end{eqnarray}
where 
\begin{eqnarray}
A^{(l,t)}     &=& I \circ F^{(l,t)} \\ 
A^{(k-1,t)} &=& \left( W^{(k)^{T}} A^{(k,t)} W^{(k)} + U^{(k-1)^{T}} A^{(k-1,t+1)} U^{(k-1)} \right) \circ F^{(k-1,t)} 
\end{eqnarray}
with binary matrix $F^{(k,t)} = \textbf{f} \phantom{'}'(\textbf{y}^{(k,t)}) \textbf{f} \phantom{'}'(\textbf{y}^{(k,t)})^{T}$ for $k=l,...,2$ and $t=\tau,...,a$, where we consider the terms for time step $t+1>\tau$ to be zero. 
\end{Corollary}

\par
Suppose that the eigenvalues of $n \times n$ Hermitian matrix have been ordered so that $\lambda_{1} \le ... \le \lambda_{n}$. Then, the eigenvalues of a sum of two Hermitian matrices satisfy
\begin{equation}
\lambda_{q}(A)+\lambda_{1}(B) \le \lambda_{q}(A + B) \le \lambda_{q}(A)+\lambda_{n}(B)
\label{SumOfEigs}
\end{equation}
see Weyl Theorem 4.3.1 in \cite{Horn1999}. Also, suppose that singular values of $n \times n$ nonsymmetric matrix have been ordered so that $\sigma_{1} \le ... \le \sigma_{n}$. Then the singular values of a sum of two nonsymmetric matrices satisfy
\begin{equation}
\sigma_{q}(A)-\sigma_{n}(B) \le \sigma_{q}(A + B) \le \sigma_{q}(A)+\sigma_{n}(B)
\label{SumOfSings}
\end{equation}
see Theorem 3.3.16 in \cite{Horn2008}, where a reverse ordering of singular values is used. 
\par
Then, the Hessian matrix eigenvalues satisfy 
\begin{equation}
    \lambda_{q}(B^{(k,t)}) + \sum_{s=1}^{t} \sum_{\zeta =1,\zeta \ne t}^{s} \lambda_{1}(B^{(k,t,s,\zeta)})
\le \lambda_{q}(H_{w}^{(k,t)}) 
\le \lambda_{q}(B^{(k,t)}) + \sum_{s=1}^{t} \sum_{\zeta =1,\zeta \ne t}^{s} \lambda_{n}(B^{(k,t,s,\zeta)})
\label{hessian_eig_bound_rnn}
\end{equation}
where auxiliary symmetric matrix
\begin{equation}
B^{(k,t,s,\zeta)} = \left\{
\begin{array}{l}
\tilde{C}^{(k,t,s,\zeta)} \otimes \left( \textbf{z}^{(k-1,s)} \textbf{z}^{(k-1,\zeta)^{T}} \right) \text{ if } s=\zeta \\
\tilde{C}^{(k,t,s,\zeta)} \otimes \left( \textbf{z}^{(k-1,s)} \textbf{z}^{(k-1,\zeta)^{T}} \right) + \tilde{C}^{(k,t,\zeta,s)} \otimes \left( \textbf{z}^{(k-1,\zeta)} \textbf{z}^{(k-1,s)^{T}} \right) \text{ otherwise }
\end{array}
\right.
\end{equation}
and we have abbreviated $B^{(k,t)}$ the term with $s=\zeta=t$. 
\par
Notice that when $s=\zeta=t$ the products with $U^{(k)}$ disappear in \eqref{hessian_cktsz_rnn} and \eqref{hessian_cktsz_rnn_simplified} and the term $B^{(k,t)} \equiv H_{w}^{(k)}$  in \eqref{hessian_hwk} and \eqref{hessian_hwk_simplified}, with all quantities taken at time $t$. Therefore, using \eqref{hessian_eigenvalue_set} we may conclude that
\begin{equation}
\lambda_{q}(B^{(k,t)}) = \left\{ 0, || \textbf{z}^{(k-1,t)} ||_{2}^{2} \lambda_{i}(A^{(k,t)}) \right\} 
\label{hessian_eigenvalue_set_rnn}
\end{equation} 
with $\lambda_{q}(B^{(k,t)})=0$ eigenvalue multiplicity being $(n-1)m$ and $i=1,...m$. Notice that similar bounds can be obtained analogously for all terms $B^{(k,t,s,\zeta)}$ with $s = \zeta$. 
\par
Also, notice that for symmetric matrix $|\lambda_{q}| = \sigma_{q}$ and that the singular values of Kronecker product of two square $m \times m$ and $n \times n$ matrices are
\begin{equation}
\sigma_{q}(A \otimes B) = \sigma_{i}(A)\sigma_{j}(B)
\label{KroneckerSings}
\end{equation}
for $i=1,...,m$, $j=1,...,n$ and $q=1,...,mn$, see Theorem 4.2.15 in \cite{Horn2008}. Therefore, using \eqref{SumOfSings} and \eqref{KroneckerSings} for the terms with $s \ne \zeta$ we obtain 
\begin{equation}
0 \le \frac{|\lambda_{q}(B^{(k,t,s,\zeta)})|}{||\textbf{z}^{(k-1,\zeta)}||_{2}||\textbf{z}^{(k-1,s)}||_{2}} \le \sigma_{q}(\tilde{C}^{(k,t,s,\zeta)})+\sigma_{n}(\tilde{C}^{(k,t,\zeta,s)})
\label{hessian_nonsym_terms_eig_bound_rnn}
\end{equation}
where we have used the fact that $\sigma_{i}(A)=\sigma_{i}(A^{T})$ and $\sigma_{n}(\textbf{v}\textbf{w}^{T}) = \sigma_{n}(\textbf{w}\textbf{v}^{T})=||\textbf{v}||_{2}||\textbf{w}||_{2}$.
\par
Finally, using \eqref{KroneckerEigs} and \eqref{hessian_eig_bound_rnn} as well as \eqref{KroneckerSings} and \eqref{hessian_nonsym_terms_eig_bound_rnn} we obtain that the Hessian matrix eigenvalues satisfy the following bounds
\begin{align}
  (t-1) \min_{s \ne t} &\phantom{.} \mu_{1}^{(k,t,s)} - t(t-1) \max_{s \ne \zeta} \sigma_{n}(\tilde{C}^{(k,t,s,\zeta)}) ||\textbf{z}^{(k-1,\zeta)}||_{2}||\textbf{z}^{(k-1,s)}||_{2}  \nonumber \\
&\le \lambda_{q}(H_{w}^{(k,t)}) - \lambda_{q}(B^{(k,t)}) \le \\
&\phantom{\le 1111111} 
  (t-1) \max_{s \ne t} \mu_{n}^{(k,t,s)} + t(t-1) \max_{s \ne \zeta} \sigma_{n}(\tilde{C}^{(k,t,s,\zeta)}) ||\textbf{z}^{(k-1,\zeta)}||_{2}||\textbf{z}^{(k-1,s)}||_{2} \nonumber
\end{align}
where $\mu_{1}^{(k,t,s)} = \min \{0, ||\textbf{z}^{(k-1,s)}||_{2}^{2} \lambda_{1}( \tilde{C}^{(k,t,s,s)}) \}$, $\mu_{n}^{(k,t,s)} = \max \{0, ||\textbf{z}^{(k-1,s)}||_{2}^{2} \lambda_{n}(\tilde{C}^{(k,t,s,s)}) \}$. 

\addtolength{\textheight}{-0.5cm} 
\clearpage
\newpage

\section{Conclusion} 

In this paper we have reviewed backward propagation, including backward propagation through time. We have shown that the weight gradient can be expressed as a rank-$1$ and rank-$t$ matrix for FNNs and RNNs respectively. Therefore, we have concluded that for mini-batch of size $r$ the weight updates based on the gradient can be expressed as rank-$rt$ matrix. Also, we have shown that for $t$ time steps the weight Hessian can be expressed as a sum of $t^{2}$ Kronecker products of rank-$1$ and $W^{T}AW$ matrices, for some matrix $A$ and weight matrix $W$. Finally, we have found an expression and bounds for the eigenvalues of the Hessian matrix in terms of smaller $m \times m$ matrices. In the future we would like to explore gradient and Hessian structure to develop novel optimization algorithms for computing weight and bias updates. 

\section{Acknowledgements}
The author would like to acknowledge Andrei Bourchtein, Ludmila Bourchtein, Thomas Breuel, Boris Ginsburg, Michael Garland, Oleksii Kuchaiev, Pavlo Molchanov and Saurav Muralidharan for their useful comments and suggestions.

\end{document}